\documentclass{article}

\usepackage[final]{neurips_2019}

\usepackage[utf8]{inputenc} %
\usepackage[T1]{fontenc}    %
\usepackage{hyperref}       %
\usepackage{url}            %
\usepackage{booktabs}       %
\usepackage{amsfonts}       %
\usepackage{nicefrac}       %
\usepackage{microtype}      %
\usepackage{amsmath}
\usepackage{amsthm}
\usepackage{amssymb}
\usepackage{bm}
\usepackage{graphicx}
\usepackage{caption}
\usepackage{subcaption}
\usepackage{xcolor}
\usepackage{algorithm}
\usepackage[noend]{algpseudocode}

\hypersetup{
    colorlinks=true,
    linkcolor=blue,
    filecolor=magenta,      
    urlcolor=cyan,
    pdftitle={Sharelatex Example},
    bookmarks=true,
    pdfpagemode=FullScreen,
    }

\newtheorem{theorem}{Theorem}

\title{Learning Imbalanced Datasets with Label-Distribution-Aware Margin Loss}

\def\shownotes{1}  %
\ifnum\shownotes=1
\newcommand{\authnote}[2]{{$\ll$\textsf{\footnotesize #1 notes: #2}$\gg$}}
\else
\newcommand{\authnote}[2]{}
\fi

\newcommand{\lda}{\ldam}
\newcommand{\BMHinge}{\textup{LDAM-HG}}
\newcommand{\ldam}{\textup{LDAM}}
\newcommand{\ldamL}{\ensuremath{\mathcal{L}_{\ldam}}}
\newcommand{\BMCE}{{\ldam}}
\newcommand{\tstagew}{DRW}
\newcommand{\tstages}{DRS}
\newcommand{\ERM}{ERM}

\author{%
  Kaidi Cao \\ 
  Stanford University\\
  \texttt{kaidicao@stanford.edu}
  \And 
  Colin Wei \\ 
  Stanford University\\
  \texttt{colinwei@stanford.edu}
  \And
  Adrien Gaidon \\
  Toyota Research Institute \\
  \texttt{adrien.gaidon@tri.global}
  \And
  Nikos Arechiga \\
  Toyota Research Institute \\
  \texttt{nikos.arechiga@tri.global}
  \And
  Tengyu Ma \\ 
  Stanford University\\
  \texttt{tengyuma@stanford.edu}
}

\begin{document}

\maketitle

\begin{abstract}
Deep learning algorithms can fare poorly when the training dataset suffers from heavy class-imbalance but the testing criterion requires good generalization on less frequent classes. We design two novel methods to improve performance in such scenarios. First, we propose a theoretically-principled label-distribution-aware margin (LDAM) loss motivated by minimizing a margin-based generalization bound. This loss replaces the standard cross-entropy objective during training and can be applied with prior strategies for training with class-imbalance such as re-weighting or re-sampling. Second, we propose a simple, yet effective, training schedule that defers re-weighting until after the initial stage, allowing the model to learn an initial representation while avoiding some of the complications associated with re-weighting or re-sampling. We test our methods on several benchmark vision tasks including the real-world imbalanced dataset iNaturalist 2018. Our experiments show that either of these methods alone can already improve over existing techniques and their combination achieves even better performance gains\footnote{Code available at \url{https://github.com/kaidic/LDAM-DRW}.}.
\end{abstract}

\section{Introduction}

Modern real-world large-scale datasets often have long-tailed label distributions~\citep{van2017devil,krishna2017visual,lin2014microsoft,everingham2010pascal,guo2016ms,thomee2015yfcc100m,liu2019large}. On these datasets, deep neural networks have been found to perform poorly on less
represented classes~\citep{he2008learning, van2017devil,buda2018systematic}. This is particularly detrimental if the testing criterion places more emphasis on minority classes. For example, accuracy on a uniform label distribution or the minimum accuracy among all classes are examples of such criteria. These are common scenarios in many applications~\citep{cao2018pose,merler2019diversity,hinnefeld2018evaluating} due to various practical concerns such as transferability to new domains, fairness, etc. 

The two common approaches for learning long-tailed data are re-weighting the losses of the examples and re-sampling the examples in the SGD mini-batch (see~\citep{buda2018systematic,huang2016learning,cui2019classbalancedloss,he2008learning,he2013imbalanced,chawla2002smote} and the references therein). They both devise a training loss that is in expectation closer to the test distribution, and therefore can achieve better trade-offs between the accuracies of the frequent classes and the minority classes. However, because we have fundamentally less information about the minority classes and the models deployed are often huge, over-fitting to the minority classes appears to be one of the challenges in improving these methods. 

We propose to regularize the minority classes more strongly than the frequent classes so that we can improve the generalization error of minority classes without sacrificing the model's ability to fit the frequent classes. Implementing this general idea requires a data-dependent or label-dependent regularizer  --- which in contrast to standard $\ell_2$ regularization depends not only on the weight matrices but also on the labels --- to differentiate frequent and minority classes. The theoretical understanding of data-dependent regularizers is sparse (see~\citep{wei2019data,nagarajan2018deterministic,arora2018stronger} for a few recent works.) 

We explore one of the simplest and most well-understood data-dependent properties: the margins of the training examples. Encouraging a large margin can be viewed as regularization, as standard generalization error bounds (e.g.,~\citep{bartlett2017spectrally,wei2018margin}) depend on the inverse of the minimum margin among all the examples. Motivated by the question of generalization with respect to minority classes, we instead study the minimum margin \textit{per class} and obtain per-class and uniform-label test error bounds.\footnote{The same technique can also be used for other test label distribution as long as the test label distribution is known. See Section~\ref{sec:imbalanced_test} for some experimental results.} Minimizing the obtained bounds gives an optimal trade-off between the margins of the classes. See Figure~\ref{fig:margin} for an illustration in the binary classification case.  

Inspired by the theory, we design a label-distribution-aware loss function that encourages the model to have the optimal trade-off between per-class margins. The proposed loss extends the existing soft margin loss~\citep{wang2018additive} by encouraging the minority classes to have larger margins.  As a label-dependent regularization technique, our modified loss function is orthogonal to the re-weighting and re-sampling approach. In fact, we also design a deferred re-balancing optimization procedure that allows us to combine the re-weighting strategy with our loss (or other losses) in a more efficient way. 

In summary, our main contributions are (i) we design a label-distribution-aware loss function to encourage larger margins for minority classes, (ii) we propose a simple deferred re-balancing optimization procedure to apply re-weighting more effectively, and (iii) our practical implementation shows significant improvements on several benchmark vision tasks, such as artificially imbalanced CIFAR and Tiny ImageNet~\citep{tinyimagenet}, and the real-world large-scale imbalanced dataset iNaturalist'18~\citep{van2018inaturalist}.

\begin{figure}
  \begin{minipage}[c]{0.47\textwidth}
    \includegraphics[width=0.9\textwidth]{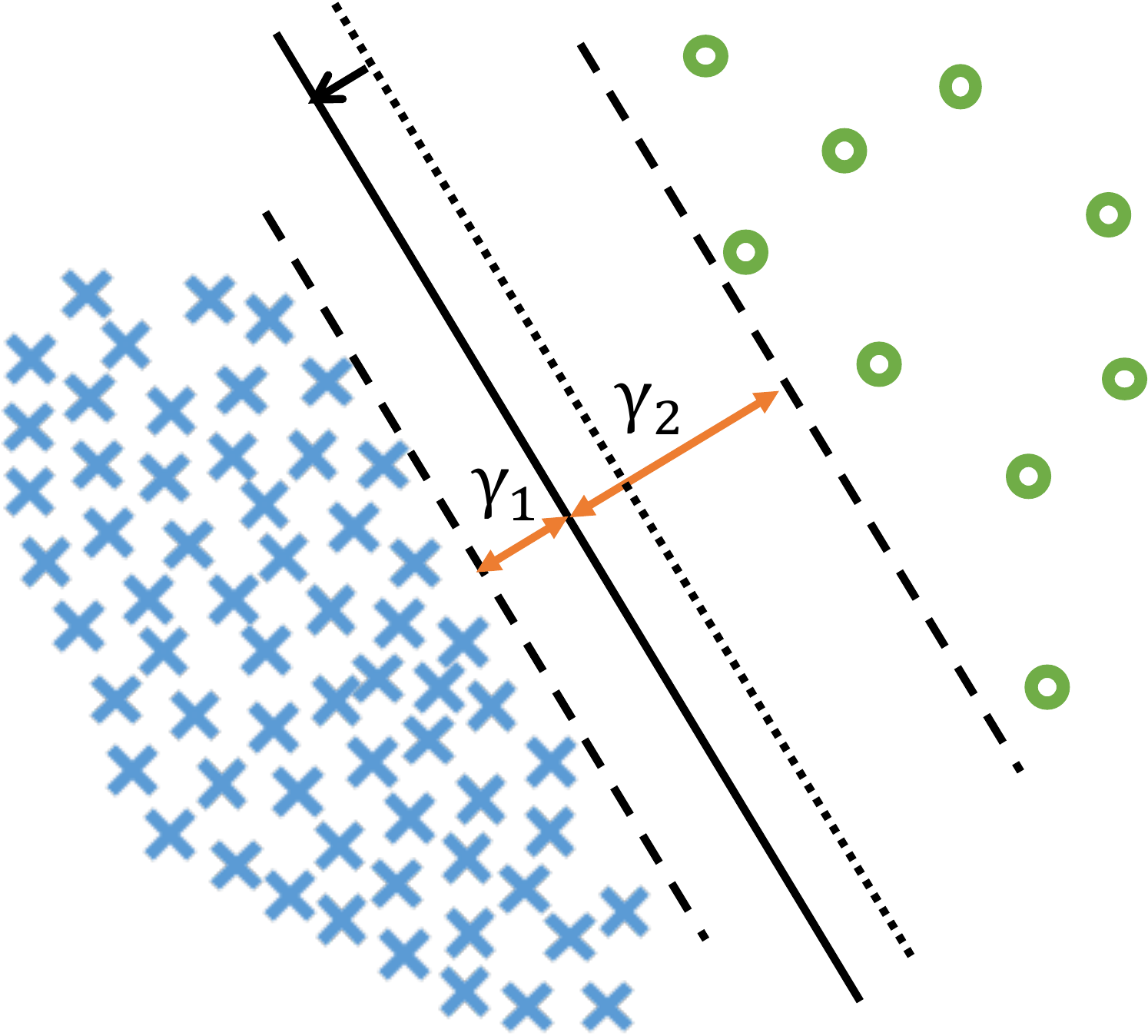}
  \end{minipage}\hfill
  \begin{minipage}[c]{0.49\textwidth}
    \caption{
       For binary classification with a linearly separable classifier, the margin $\gamma_i$ of the $i$-th class is defined to be the the minimum distance of the data in the $i$-th class to the decision boundary. We show that the test error with the uniform label distribution is bounded by a quantity that scales in $\frac{1}{\gamma_1\sqrt{n_1}} + \frac{1}{\gamma_2\sqrt{n_2}}$. 
    As illustrated here, fixing the direction of the decision boundary leads to a fixed $\gamma_1+\gamma_2$, but the trade-off between $\gamma_1,\gamma_2$ can be optimized by shifting the decision boundary. As derived in Section~\ref{sec:theory}, the optimal trade-off is $\gamma_i \propto n_i^{-1/4}$ where $n_i$ is the sample size of  the $i$-th class.
    } \label{fig:margin}	
  \end{minipage}
\end{figure}

\section{Related Works}

Most existing algorithms for learning imbalanced datasets can be divided in to two categories: re-sampling and re-weighting.

\textbf{Re-sampling.} There are two types of re-sampling techniques: over-sampling the minority classes (see e.g., \citep{shen2016relay,zhong2016towards,buda2018systematic,byrd2019what} and references therein) and under-sampling the frequent classes (see, e.g., \citep{he2008learning,japkowicz2002class,buda2018systematic} and the references therein.)
The downside of under-sampling is that it discards a large portion of the data and thus is not feasible when data imbalance is extreme. Over-sampling is effective in a lot of cases but can lead to over-fitting of the minority classes~\citep{chawla2002smote,cui2019classbalancedloss}. Stronger data augmentation for minority classes can help alleviate the over-fitting~\citep{chawla2002smote,zou2018domain}.

\textbf{Re-weighting.} Cost-sensitive re-weighting assigns (adaptive) weights for different classes or even different samples. The vanilla scheme re-weights classes proportionally to the inverse of their frequency \citep{huang2016learning,huang2019deep,wang2017learning}. 
Re-weighting methods tend to make the optimization of deep models difficult under extreme data imbalanced settings and large-scale scenarios~\citep{huang2016learning,huang2019deep}.
\citet{cui2019classbalancedloss} observe that re-weighting by inverse class frequency yields poor performance on frequent classes, and thus propose re-weighting by the inverse effective number of samples. 
This is the main prior work that we empirically compare with. 

Another line of work assigns weights to each sample based on their individual properties. Focal loss \citep{lin2017focal} down-weights the well-classified examples; \citet{li2018gradient} suggests an improved technique which down-weights examples with either very small gradients or large gradients because examples with small gradients are well-classified and those with large gradients tend to be outliers.

In a recent work~\citep{byrd2019what}, Byrd and Lipton study the effect of importance weighting and show that empirically importance weighting does not have a significant effect when no regularization is applied, which is consistent with the theoretical prediction in~\citep{soudry2018implicit} that logistical regression without regularization converges to the max margin solution. In our work, we explicitly encourage rare classes to have higher margin, and therefore we don't converge to a max margin solution. Moreover, in our experiments, we apply non-trivial $\ell_2$-regularization to achieve the best generalization performance. We also found deferred re-weighting (or deferred re-sampling) are more effective than re-weighting and re-sampling from the beginning of the training. 

In contrast, and orthogonally to these papers above, our main technique aims to improve the generalization of the minority classes by applying additional regularization that is orthogonal to the re-weighting scheme. We also propose a deferred re-balancing optimization procedure to improve the optimization and generalization of a generic re-weighting scheme. %

\textbf{Margin loss.} The hinge loss is often used to obtain a ``max-margin'' classifier, most notably in SVMs~\citep{suykens1999least}. Recently, Large-Margin Softmax~\citep{liu2016large}, Angular Softmax~\citep{liu2017sphereface}, and Additive Margin Softmax~\citep{wang2018additive} have been proposed to minimize intra-class variation in predictions and enlarge the inter-class margin by incorporating the idea of angular margin. In contrast to the class-independent margins in these papers, our approach encourages bigger margins for minority classes. Uneven margins for imbalanced datasets are also proposed and studied in~\citep{li2002perceptron} and the recent work~\citep{khan2019striking,li2019overfitting}. Our theory put this idea on a more theoretical footing by providing a concrete formula for the desired margins of the classes alongside good empirical progress.

\textbf{Label shift in domain adaptation.} The problem of learning imbalanced datasets can be also viewed as a label shift problem in transfer learning or domain adaptation (for which we refer the readers to the survey~\citep{wang2018deep} and the reference therein).  In a typical label shift formulation, the difficulty is to detect and estimate the label shift, and after estimating the label shift, re-weighting or re-sampling is applied. We are addressing a largely different question: can we do better than re-weighting or re-sampling when the label shift is known?  In fact, our algorithms can be used to replace the re-weighting steps of some of the recent interesting work on detecting and correcting label shift~\citep{lipton2018detecting, azizzadenesheli2018regularized}. 

Distributionally robust optimization (DRO) is another technique for domain adaptation (see~\citep{duchidistributionally,hashimoto2018fairness,carmon2019variance} and the reference therein.) However, the formulation assumes no knowledge of the target label distribution beyond a bound on the amount of shift, which makes the problem very challenging. We here assume the knowledge of the test label distribution, using which we design efficient methods that can scale easily to large-scale vision datasets with significant improvements. 

\noindent{\bf Meta-learning.} Meta-learning is also used in improving the performance on imbalanced datasets or the few shot learning settings. We refer the readers to~\citep{wang2017learning,shu2019meta,wang2018low} and the references therein. So far, we generally believe that our approaches that modify the losses are more computationally efficient than meta-learning based approaches.

\section{Main Approach}

\newcommand{\R}{\mathbb{R}}
\renewcommand{\P}{\mathcal{P}}
\newcommand{\pdata}{\P}
\newcommand{\E}{\mathbb{E}}
\newcommand{\Lerror}{L_{\textup{bal}}}

\newcommand{\comp}{\textup{C}}
\newcommand{\cF}{\mathcal{F}}
\subsection{Theoretical Motivations}\label{sec:theory}

\paragraph{Problem setup and notations.}
We assume the input space is $\R^d$ and the label space is $\{1,\dots, k\}$. Let $x$ denote the input and $y$ denote the corresponding label. We assume that the class-conditional distribution $\P(x\mid y)$ is the same at training and test time. Let $\P_j$ denote the class-conditional distribution, i.e. $\P_j = \P(x \mid y = j)$. We will use $\P_{\textup{bal}}$ to denote the balanced test distribution which first samples a class uniformly and then samples data from $\P_j$. 

For a model $f:\R^d \rightarrow \R^k$ that outputs $k$ logits, we use $\Lerror[f]$ to denote the standard 0-1 test error on the balanced data distribution:
\begin{align*}
\Lerror[f] = \Pr_{(x,y)\sim \P_{\textup{bal}}} [f(x)_y < \max_{\ell\ne y} f(x)_\ell]
\end{align*}

Similarly, the error $L_j$ for class $j$ is defined as $L_j[f] = \Pr_{(x,y)\sim \P_{j}} [f(x)_y < \max_{\ell\ne y} f(x)_\ell]$. Suppose we have a training dataset $\{(x_i, y_i)\}_{i=1}^n$.  Let $n_j$ be the number of examples in class $j$. Let $S_j = \{i : y_i = j\}$ denote the example indices corresponding to class $j$.

Define the margin of an example $(x,y)$ as 
\begin{align}
\gamma(x,y) = f(x)_y - \max_{j\neq y} f(x)_j
\end{align}

Define the training margin for class $j$ as:
\begin{align}
\gamma_j = \min_{i \in S_j }\gamma(x_i, y_i) 
\end{align}
We  consider the separable cases (meaning that all the training examples are classified correctly) because neural networks are often over-parameterized and can fit the training data well. 
We also note that the minimum margin of all the classes,  $\gamma_{\min} = \min\{\gamma_1,\dots, \gamma_k\}$, is the classical notion of training margin studied in the past~\citep{koltchinskii2002empirical}.

\paragraph{Fine-grained generalization error bounds.}
Let $\cF$ be the family of hypothesis class. Let $\comp(\cF)$ be some proper complexity measure of the hypothesis class $\cF$. There is a large body of recent work on measuring the complexity of neural networks (see~\citep{bartlett2017spectrally,golowich2017size,wei2019data} and references therein), and our discussion below is orthogonal to the precise choices. When the training distribution and the test distribution are the same, the typical generalization error bounds scale in $\comp(\cF)/\sqrt{n}$. That is, in our case, if the test distribution is also imbalanced as the training distribution, then %
\begin{align}
\textup{imbalanced test error} \lesssim \frac{1}{\gamma_{\min}} \sqrt{\frac{\comp(\cF)}{n}}
\end{align} 
Note that the bound is oblivious to the label distribution, and only involves the minimum margin across all examples and the total number of data points. We extend such bounds to the setting with balanced test distribution by considering the margin of each class. As we will see, the more fine-grained bound below allows us to design new training loss function that is customized to the imbalanced dataset. 

\begin{theorem}[Informal and simplified version of Theorem~\ref{thm:technical_main}]\label{thm:informal_main}
	With high probability ($1-n^{-5}$) over the randomness of the training data, the error $L_j$ for class $j$ is bounded by 
	\begin{align}
	L_j[f] \lesssim \frac{1}{\gamma_j}  \sqrt{\frac{\comp(\cF) }{n_j}} + \frac{\log n}{\sqrt{n_j}} \label{eqn:perclass}
	\end{align}
	where we use $\lesssim$ to hide constant factors. As a direct consequence, 
	\begin{align}
	\Lerror[f] \lesssim \frac{1}{k}\sum_{j= 1}^k \left(\frac{1}{\gamma_j}  \sqrt{\frac{\comp(\cF) }{n_j}} + \frac{\log n}{\sqrt{n_j}}\right)\label{eqn:bound}
	\end{align}
\end{theorem}

\paragraph{Class-distribution-aware margin trade-off.}
The generalization error bound~\eqref{eqn:perclass} for each class suggests that if we wish to improve the generalization of minority classes (those with small $n_j$'s), we should aim to enforce bigger margins $\gamma_j$'s for them. However, enforcing bigger margins for minority classes may hurt the margins of the frequent classes. What is the optimal trade-off between the margins of the classes? An answer for the general case may be difficult, but fortunately we can obtain the optimal trade-off for the binary classification problem. 

With $k=2$ classes, we aim to optimize the balanced generalization error bound provided in~\eqref{eqn:bound}, which can be simplified to (by removing the low order term $\frac{\log n}{\sqrt{n_j}}$ and the common factor $\comp(\cF)$)
\begin{align}
\frac{1}{\gamma_1\sqrt{n_1}} + \frac{1}{\gamma_2\sqrt{n_2}}\label{eqn:two-class}
\end{align}
At the first sight, because $\gamma_1$ and $\gamma_2$ are complicated functions of the weight matrices, it appears difficult to understand the optimal margins. However, we can figure out the relative scales between $\gamma_1$ and $\gamma_2$.  Suppose $\gamma_1,\gamma_2 > 0$ minimize the equation above, we observe that any $\gamma_1' = \gamma_1-\delta$ and $\gamma_2'=\gamma_2 + \delta$ (for $\delta \in (-\gamma_2, \gamma_1)$) can be realized by the same weight matrices with a shifted bias term (See Figure~\ref{fig:margin} for an illustration). Therefore, for $\gamma_1,\gamma_2$ to be optimal, they should satisfy
\begin{align}
\frac{1}{\gamma_1\sqrt{n_1}} + \frac{1}{\gamma_2\sqrt{n_2}} \ge \frac{1}{(\gamma_1-\delta)\sqrt{n_1}} + \frac{1}{(\gamma_2+\delta)\sqrt{n_2}}
\end{align} 
The equation above implies that 
\begin{align}
\gamma_1 = \frac{C}{n_1^{1/4}}, \textup{ and } \gamma_2 = \frac{C}{n_2^{1/4}}
\end{align}
for some constant $C$. Please see a detailed derivation in the Section~\ref{sec:missing_theory}. 

\noindent{\bf Fast rate vs slow rate, and the implication on the choice of margins.} The bound in Theorem~\ref{thm:informal_main} may not necessarily be tight. The generalization bounds that scale in $1/\sqrt{n}$ (or $1/\sqrt{n_i}$ here with imbalanced classes) are generally referred to the ``slow rate'' and those that scale in $1/n$ are referred to the ``fast rate''. With deep neural networks and when the model is sufficiently big enough, it is possible that some of these bounds can be improved to the fast rate. See~\citep{wei2019improved} for some recent development. In those cases, we can derive the optimal trade-off of the margin to be $n_i \propto n_i^{-1/3}$.

\subsection{Label-Distribution-Aware Margin Loss} \label{sec:loss}

Inspired by the trade-off between the class margins in Section~\ref{sec:theory} for two classes, we propose to enforce a class-dependent margin for multiple classes of the form 
\begin{align}
\gamma_j = \frac{C}{n_{j}^{1/4}} 
\end{align}

We will design a soft margin loss function to encourage the network to have the margins above. Let $(x,y)$ be an example and $f$ be a model. For simplicity, we use $z_j = f(x)_j$ to denote the $j$-th output of the model for the $j$-th class. 

The most natural choice  would be a multi-class extension of the hinge loss:
\begin{align}
\mathcal{L}_{\BMHinge{}}((x,y); f) & = \max( \max_{j \neq y} \{ z_j \} - z_{y} + \Delta_{y}, 0)\\ 
& \textup{where } \Delta_j = \frac{C}{n_{j}^{1/4}} \textup{ for } j \in \{1,\dots, k\}
\end{align}

Here $C$ is a hyper-parameter to be tuned. In order to tune the margin more easily, we effectively normalize the logits (the input to the loss function) by normalizing last hidden activation to $\ell_2$ norm 1, and normalizing the weight vectors of the last fully-connected layer to $\ell_2$ norm 1, following the previous work~\citep{wang2018additive}. Empirically, the non-smoothness of hinge loss may pose difficulties for optimization. The smooth relaxation of the hinge loss is the following cross-entropy loss with enforced margins: 

\begin{align}
\ldamL((x,y);f) & = -\log \frac{e^{z_{y}-\Delta_{y}}}{e^{z_{y}-\Delta_{y}} + \sum_{j \neq y} e^{z_j}} \\
& \textup{where } \Delta_j = \frac{C}{n_{j}^{1/4}} \textup{ for } j \in \{1,\dots, k\} \label{eqn:delta}
\end{align}

In the previous work~\citep{liu2016large,liu2017sphereface,wang2018additive} where the training set is usually balanced, the margin $\Delta_{y}$ is chosen to be a label independent constant $C$, whereas our margin depends on the label distribution. 

\textit{Remark:} Attentive readers may find the loss  $\ldamL$ somewhat reminiscent of the re-weighting because in the binary classification case --- where the model outputs a single real number which is passed through a sigmoid to be converted into a probability, ---  both the two approaches change the gradient of an example by a scalar factor. However, we remark two key differences: the scalar factor introduced by the re-weighting only depends on the class, whereas the scalar introduced by $\ldamL$ also depends on the output of the model; for multiclass classification problems, the proposed loss $\ldamL$ affects the gradient of the example in a more involved way than only introducing a scalar factor. Moreover, recent work has shown that, under separable assumptions, the logistical loss, with weak regularization~\citep{wei2018margin} or without regularization~\citep{soudry2018implicit}, gives the max margin solution, which is in turn not effected by any re-weighting by its definition. This further suggests that the loss $\ldamL$ and the re-weighting may complement each other, as we have seen in the experiments. (Re-weighting would affect the margin in the non-separable data case, which is left for future work.)

\subsection{Deferred Re-balancing Optimization Schedule} \label{sec:schedule}

Cost-sensitive re-weighting and re-sampling are two well-known and successful strategies to cope with imbalanced datasets because, in expectation, they effectively make the imbalanced training distribution closer to the uniform test distribution. The known issues with applying these techniques are (a) re-sampling the examples in minority classes often causes heavy over-fitting to the minority classes when the model is a deep neural network, as pointed out in prior work (e.g.,~\citep{cui2019classbalancedloss}), and (b) weighting up the minority classes' losses can cause difficulties and instability in optimization, especially when the classes are extremely imbalanced~\citep{cui2019classbalancedloss, huang2016learning}. In fact, ~\citet{cui2019classbalancedloss} develop a novel and sophisticated learning rate schedule to cope with the optimization difficulty.

We observe empirically that re-weighting and re-sampling are both inferior to the vanilla empirical risk minimization (ERM) algorithm (where all training examples have the same weight) before annealing the learning rate in the following sense. The features produced before annealing the learning rate by re-weighting and re-sampling are worse than those produced by ERM. (See Figure~\ref{fig:linear_finetune} for an ablation study of the feature quality performed by training linear classifiers on top of the features on a large balanced dataset.)

Inspired by this, we develop a deferred re-balancing training procedure (Algorithm~\ref{alg:final}), which first trains using vanilla ERM with the \lda{} loss before annealing the learning rate, and then deploys a re-weighted \lda{} loss with a smaller learning rate. Empirically, the first stage of training leads to a good initialization for the second stage of training with re-weighted losses. Because the loss is non-convex and the learning rate in the second stage is relatively small, the second stage does not move the weights very far. Interestingly, with our \lda{} loss and deferred re-balancing training, the vanilla re-weighting scheme (which re-weights by the inverse of the number of examples in each class) works as well as the re-weighting scheme introduced in prior work~\citep{cui2019classbalancedloss}. We also found that with our re-weighting scheme and LDAM, we are less sensitive to early stopping than~\citep{cui2019classbalancedloss}.

\newcommand{\cD}{\mathcal{D}}
\newcommand{\cB}{\mathcal{B}}

\begin{algorithm}[H]
\caption{Deferred Re-balancing Optimization with \lda{} Loss} \label{alg:final}
\begin{algorithmic}[1]

\Require Dataset $\mathcal{D} = \{(x_i,y_i)\}_{i=1}^n$. A parameterized model $f_\theta$

\State Initialize the model parameters $\theta$ randomly
\For {$t=1$ to $T_0$}
	\State  $\cB\leftarrow \text{SampleMiniBatch}(\mathcal{D}, m)$  \Comment{a mini-batch of $m$ examples}
    \State $\mathcal{L}(f_\theta) \leftarrow \frac{1}{m} \sum_{(x,y)\in \cB}\ldamL((x,y);f_\theta) $ 
    \State $f_\theta \leftarrow f_\theta - \alpha\nabla_\theta \mathcal{L}(f_\theta)$ \Comment one SGD step
    \State Optional: $\alpha \leftarrow  \alpha/\tau$  \Comment anneal learning rate by a factor $\tau$ if necessary
\EndFor

\State

\For {$t=T_0$ to $T$ }
	\State  $\cB \leftarrow \text{SampleMiniBatch}(\mathcal{D}, m)$  \Comment{A mini-batch of $m$ examples}
       \State $\mathcal{L}(f_\theta) \leftarrow \frac{1}{m}\sum_{(x,y)\in \cB} n_y^{-1}\cdot \ldamL((x,y);f_\theta) $ \Comment{standard re-weighting by frequency}
    \State $f_\theta \leftarrow f_\theta -\alpha \frac{1}{\sum_{(x,y)\in \cB} n_y^{-1}}\nabla_\theta \mathcal{L}(f_\theta)$ \Comment one SGD step with re-normalized learning rate
\EndFor
\end{algorithmic}
\end{algorithm}

\section{Experiments} \label{sec:experiments}

We evaluate our proposed algorithm on artificially created versions of IMDB review~\citep{maas2011learning}, CIFAR-10, CIFAR-100 \citep{krizhevsky2009learning}  and Tiny ImageNet \citep{russakovsky2015imagenet,tinyimagenet} with controllable degrees of data imbalance, as well as a real-world large-scale imbalanced dataset, iNaturalist 2018 \citep{van2018inaturalist}. Our core algorithm is developed using PyTorch \citep{paszke2017automatic}.
\paragraph{Baselines.}
We compare our methods with the standard training and several state-of-the-art techniques and their combinations that have been widely adopted to mitigate the issues with training on imbalanced datasets: (1) Empirical risk minimization (ERM) loss: all the examples have the same weights; by default, we use standard cross-entropy loss.  (2) Re-Weighting (RW): we re-weight each sample by the inverse of the sample size of its class, and then re-normalize to make the weights 1 on average in the mini-batch. (3) Re-Sampling (RS): each example is sampled with probability proportional to the inverse sample size of its class. (4) CB~\citep{cui2019classbalancedloss}: the examples are re-weighted or re-sampled according to the inverse of the effective number of samples in each class, defined as $(1-\beta^{n_i})/(1-\beta)$, instead of inverse class frequencies. This idea can be combined with either re-weighting or re-sampling. (5) Focal: we use the recently proposed focal loss~\citep{lin2017focal} as another baseline. (6) SGD schedule: by SGD, we refer to the standard schedule where the learning rates are decayed a constant factor at certain steps; we use a standard learning rate decay schedule.

\paragraph{Our proposed algorithm and variants.} We test combinations of the following techniques proposed by us. (1) \tstagew{} and \tstages{}: following the proposed training Algorithm~\ref{alg:final}, we use the standard \ERM{} optimization schedule until the last learning rate decay, and then apply re-weighting or re-sampling for optimization in the second stage. (2) \BMCE{}: the proposed Label-Distribution-Aware Margin losses as described in Section~\ref{sec:loss}.

When two of these methods can be combined, we will concatenate the acronyms with a dash in between as an abbreviation. The main algorithm we propose is \BMCE{}-\tstagew{}. Please refer to Section~\ref{sec:implementation} for additional implementation details. 

\subsection{Experimental results on IMDB review dataset}\label{sec:IMDB}

IMDB review dataset consists of 50,000 movie reviews for binary sentiment classification~\citep{maas2011learning}. The original dataset contains an evenly distributed number of positive and negative reviews. We manually created an imbalanced training set by removing 90\% of negative reviews. We train a two-layer bidirectional LSTM with Adam optimizer~\citep{kingma2014adam}. The results are reported in Table~\ref{tab:lmdb-table}.

\begin{table}[]
	\caption{Top-1 validation errors on imbalanced IMDB review dataset. Our proposed approach LDAM-DRW outperforms the baselines.}
	\label{tab:lmdb-table}
	\centering
	\begin{tabular}{c|ccc}
		\toprule
		Approach & Error on positive reviews&  Error on negative reviews& Mean Error \\
		\midrule
		ERM & 2.86 & 70.78 & 36.82 \\
		RS & 7.12 & 45.88 & 26.50 \\
		RW & 5.20 & 42.12 & 23.66 \\
		LDAM-DRW & 4.91 & 30.77 & 17.84 \\
		\bottomrule
	\end{tabular}
\end{table}

\subsection{Experimental results on CIFAR}\label{sec:cifar}

\begin{table}[]
	\caption{Top-1 validation errors of ResNet-32 on imbalanced CIFAR-10 and CIFAR-100. The combination of our two techniques, \BMCE{}-\tstagew{}, achieves the best performance, and each of them individually are beneficial when combined with other losses or schedules.}
	\label{tab:margin-table}
	\centering
	\begin{tabular}{c|cc|cc|cc|cc}
		\toprule
		Dataset           & \multicolumn{4}{c|}{Imbalanced CIFAR-10}                               & \multicolumn{4}{c}{Imbalanced CIFAR-100}                              \\ \midrule
		Imbalance Type         & \multicolumn{2}{c|}{long-tailed} & \multicolumn{2}{c|}{step} & \multicolumn{2}{c|}{long-tailed} & \multicolumn{2}{c}{step} \\ \midrule
		Imbalance Ratio        & \multicolumn{1}{c|}{100}  & 10 & \multicolumn{1}{c|}{100}   & 10   & \multicolumn{1}{c|}{100}  & 10 & \multicolumn{1}{c|}{100}   & 10   \\ \midrule
		\ERM{} & 29.64& 13.61& 36.70 & 17.50 & 61.68& 44.30 & 61.45& 45.37 \\
		Focal~\citep{lin2017focal} & 29.62 & 13.34 & 36.09 & 16.36 & 61.59 & 44.22 &  61.43                 & 46.54 \\
		\BMCE{} & 26.65& 13.04& 33.42& 15.00  & 60.40 & 43.09& 60.42& 43.73 \\ 
		\midrule \midrule
		CB RS &  29.45  &  13.21 & 38.14 & 15.41  & 66.56& 44.94  & 66.23& 46.92 \\
	    CB RW~\citep{cui2019classbalancedloss} & 27.63& 13.46& 38.06 & 16.20 & 66.01& 42.88 & 78.69 & 47.52 \\
		CB Focal~\citep{cui2019classbalancedloss} & 25.43 & 12.90 & 39.73 & 16.54 & 63.98 & 42.01 & 80.24 & 49.98 \\
		\midrule
		HG-\tstages{}&  27.16 & 14.03 &   29.93  & 14.85 & - & - & - & - \\
		\BMHinge{}-\tstages{}& 24.42  & 12.72 & 24.53  & 12.82 & - & - & - & - \\
		M-\tstagew{}& 24.94 & 13.57 & 27.67 & 13.17 & 59.49  & 43.78 & 58.91 & 44.72\\
		\textbf{\BMCE{}-\tstagew{}}& \textbf{22.97}& \textbf{11.84}& \textbf{23.08}& \textbf{12.19}& \textbf{57.96}& \textbf{41.29}& \textbf{54.64}& \textbf{40.54} \\
		\bottomrule
	\end{tabular}
\end{table}

\textbf{Imbalanced CIFAR-10 and CIFAR-100.} The original version of CIFAR-10 and CIFAR-100 contains 50,000 training images and 10,000 validation images of size $32\times 32$ with 10 and 100 classes, respectively. To create their imbalanced version, we reduce the number of training examples per class and keep the validation set unchanged. To ensure that our methods apply to a variety of settings, we consider two types of imbalance: long-tailed imbalance \citep{cui2019classbalancedloss} and step imbalance \citep{buda2018systematic}. We use imbalance ratio $\rho$ to denote the ratio between sample sizes of the most frequent and least frequent class, i.e., $\rho = \max_i \{n_i\} / \min_i \{n_i\}$. Long-tailed imbalance follows an exponential decay in sample sizes across different classes. For step imbalance setting, all minority classes have the same sample size, as do all frequent classes. This gives a clear distinction between minority classes and frequent classes, which is particularly useful for ablation study. We further define the fraction of minority classes as $\mu$. By default we set $\mu = 0.5$ for all experiments. 

We report the top-1 validation error of various methods for imbalanced versions of CIFAR-10 and CIFAR-100 in Table~\ref{tab:margin-table}. Our proposed approach is \ldam{}-\tstagew{}, but we also include a various combination of our two techniques with other losses and training schedule for our ablation study. 

We first show that the proposed label-distribution-aware margin cross-entropy loss is superior to pure cross-entropy loss and one of its variants tailored for imbalanced data, focal loss, while no data-rebalance learning schedule is applied. We also demonstrate that our full pipeline outperforms the previous state-of-the-arts by a large margin. 
To further demonstrate that the proposed LDAM loss is essential, we compare it with regularizing by a uniform margin across all classes under the setting of cross-entropy loss and hinge loss. We use M-DRW to denote the algorithm that uses a cross-entropy loss with uniform margin~\citep{wang2018additive} to replace \ldam, namely, the $\Delta_j$ in equation~\eqref{eqn:delta} is chosen to be a tuned constant that does not depend on the class $j$. Hinge loss (HG) suffers from optimization issues with 100 classes so we constrain its experiment setting with CIFAR-10 only.

\textbf{Imbalanced but known test label distribution: } We also test the performance of an extension of our algorithm in the setting where the test label distribution is known but not uniform. Please see Section~\ref{sec:imbalanced_test} for details.

\subsection{Visual recognition on iNaturalist 2018 and imbalanced Tiny ImageNet}

\begin{table}[]
	\caption{Validation errors on iNaturalist 2018 of various approaches. Our proposed method \ldam{}-\tstagew{} demonstrates significant improvements over the previous state-of-the-arts. We include ERM-\tstagew{} and \BMCE{}-SGD for the ablation study.}
	\label{tab:inat-table}
	\centering
	\begin{tabular}{cc|cc}
		\toprule
		Loss       & Schedule & Top-1 & Top-5 \\
		\midrule
ERM & SGD        & 42.86 & 21.31 \\
		CB Focal~\citep{cui2019classbalancedloss} & SGD       & 38.88 & 18.97 \\
ERM & \tstagew{} & 36.27 & 16.55 \\
		\BMCE{} & SGD & 35.42 & 16.48 \\
	    \textbf{\BMCE{}} & \textbf{\tstagew{}} & \textbf{32.00} & \textbf{14.82} \\
		\bottomrule
	\end{tabular}
\end{table}

We further verify the effectiveness of our method on large-scale imbalanced datasets. The iNatualist species classification and detection dataset \citep{van2018inaturalist} is a real-world large-scale imbalanced dataset which has 437,513 training images with a total of 8,142 classes in its 2018 version. We adopt the official training and validation splits for our experiments. The training datasets have a long-tailed label distribution and the validation set is designed to have a balanced label distribution. We use ResNet-50 as the backbone network across all experiments for iNaturalist 2018. Table~\ref{tab:inat-table} summarizes top-1 validation error for iNaturalist 2018. Notably, our full pipeline is able to outperform the ERM baseline by 10.86\% and previous state-of-the-art by 6.88\% in top-1 error. Please refer to Appendix~\ref{sec:tinyimagenet} for results on imbalanced Tiny ImageNet.

\subsection{Ablation study} \label{sec:ablation}

\begin{figure}
\begin{minipage}[b]{0.48\linewidth}
\centering
\includegraphics[width=\textwidth]{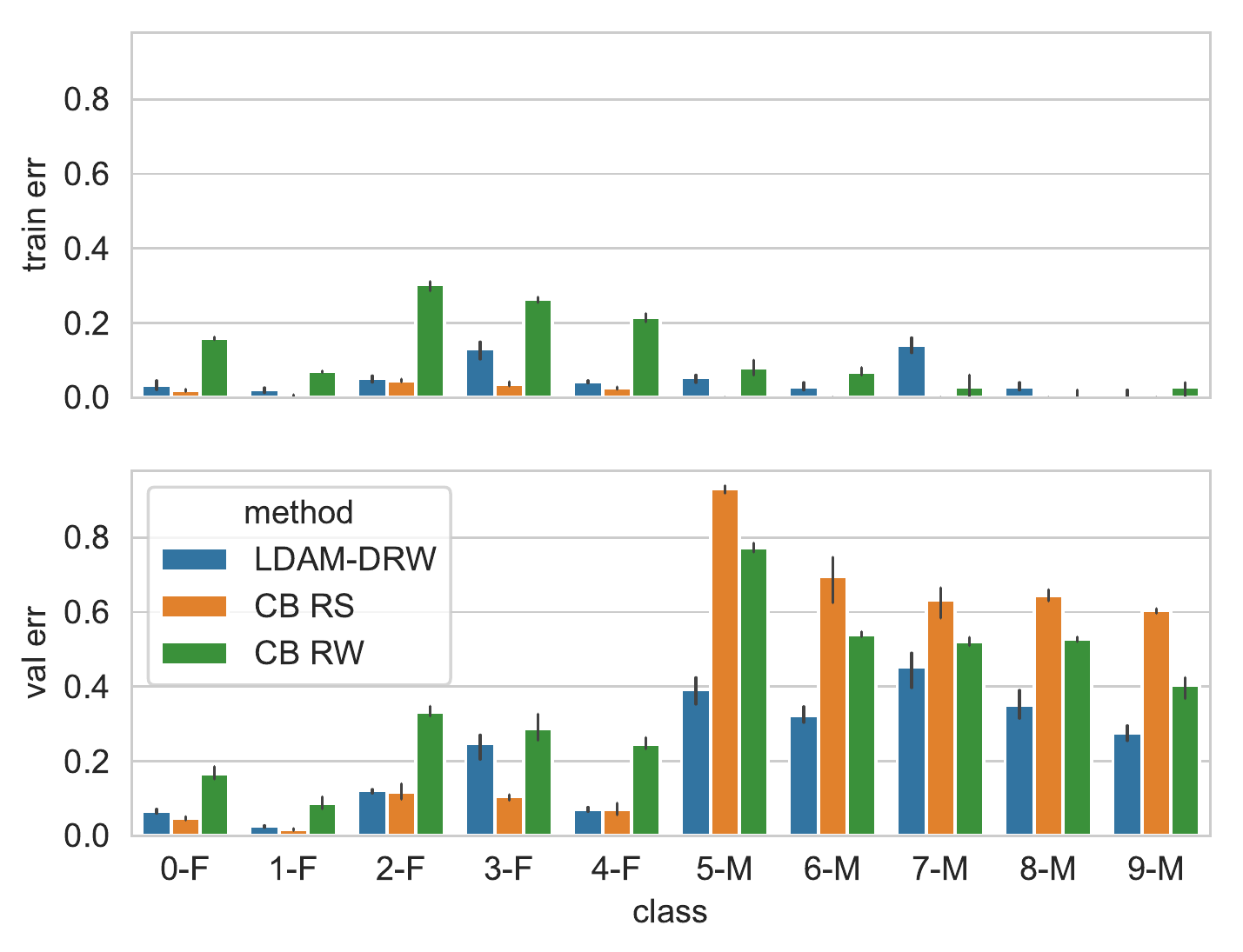}
\caption{Per-class top-1 error on CIFAR-10 with step imbalance ($\rho=100, \mu=0.5$). Classes 0-F to 4-F are frequent classes, and the rest are minority classes. Under this extremely imbalanced setting RW suffers from under-fitting, while RS over-fits on minority examples. On the contrary, the proposed algorithm exhibits great generalization on minority classes while keeping the performance on frequent classes almost unaffected. This suggests we succeeded in regularizing minority classes more strongly.}
\label{fig:per_cls_acc}
\end{minipage}
\hspace{0.5cm}
\begin{minipage}[b]{0.48\linewidth}
\centering
\includegraphics[width=\textwidth]{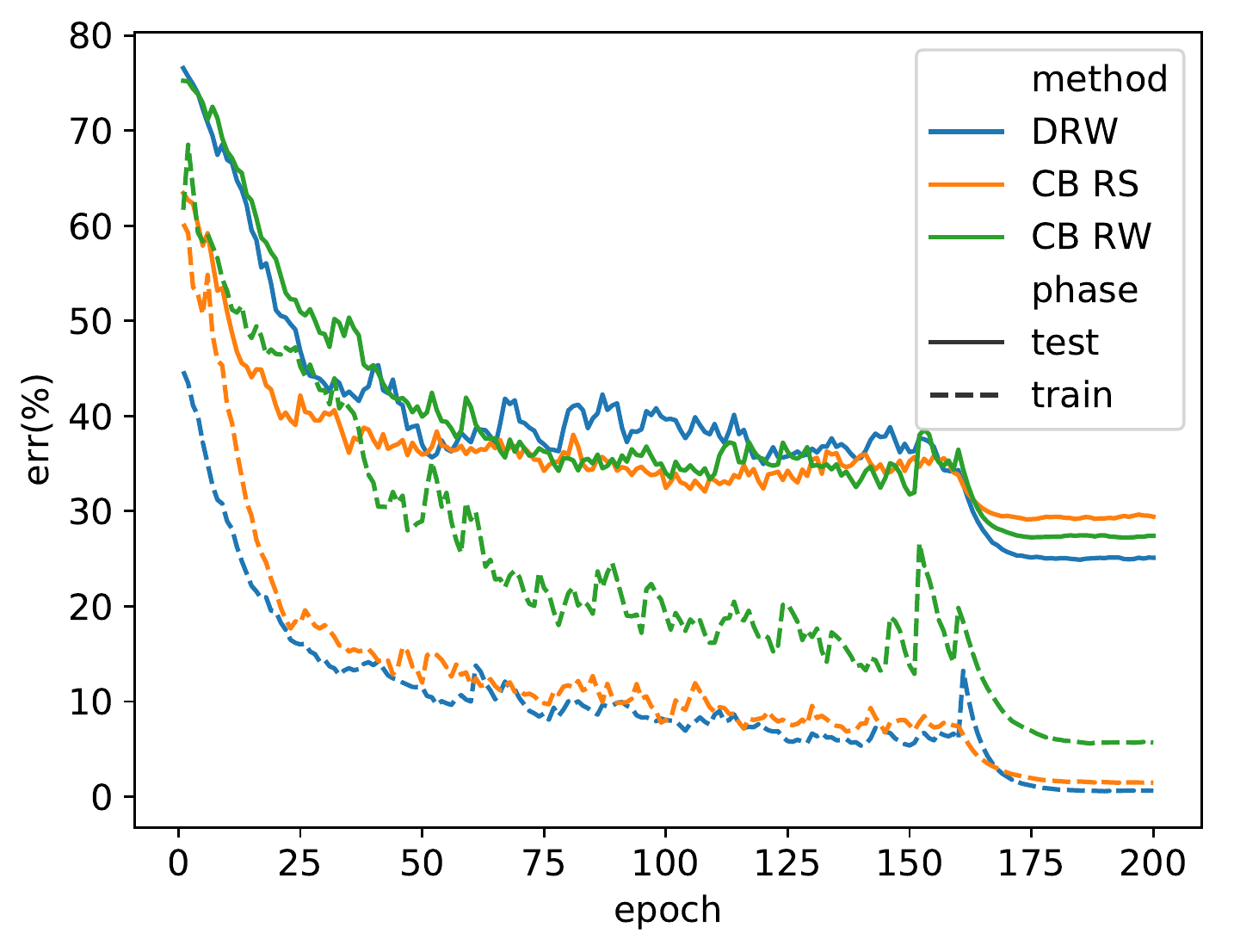}
\caption{Imbalanced training errors (dotted lines) and \textit{balanced} test errors (solid lines) on CIFAR-10 with long-tailed imbalance ($\rho=100$). We anneal decay the learning rate at epoch 160 for all algorithms. Our \tstagew{} schedule uses ERM before annealing the learning rate and thus performs worse than RW and RS before that point, as expected. However, it outperforms the others significantly after annealing the learning rate. See Section~\ref{sec:ablation} for more analysis. {\color{white} xxxxxxxxxxxxxxxxxxxxxxxxxxxxxxxxxxxxx}}
\label{fig:exp_plot}
\end{minipage}
\end{figure}

\textbf{Evaluating generalization on minority classes.} To better understand the improvement of our algorithms, we show per-class errors of different methods in Figure~\ref{fig:per_cls_acc} on imbalanced CIFAR-10. Please see the caption there for discussions.%

\textbf{Evaluating deferred re-balancing schedule.} We compare the learning curves of deferred re-balancing schedule with other baselines in Figure~\ref{fig:exp_plot}. In Figure~\ref{fig:linear_finetune} of Section~\ref{sec:feature_extractor}, we further show that even though ERM in the first stage has slightly worse or comparable balanced test error compared to RW and RS, in fact the features (the last-but-one layer activations) learned by ERM are better than those by RW and RS. This agrees with our intuition that the second stage of \tstagew{}, starting from better features, adjusts the decision boundary and locally fine-tunes the features.%

\section{Conclusion}

We propose two methods for training on imbalanced datasets, label-distribution-aware margin loss (LDAM), and a deferred re-weighting (DRW) training schedule. Our methods achieve significantly improved performance on a variety of benchmark vision tasks. Furthermore, we provide a theoretically-principled justification of LDAM by showing that it optimizes a uniform-label generalization error bound. For DRW, we believe that deferring re-weighting lets the model avoid the drawbacks associated with re-weighting or re-sampling until after it learns a good initial representation (see some analysis in Figure~\ref{fig:exp_plot} and Figure~\ref{fig:linear_finetune}). However, the precise explanation for DRW's success is not fully theoretically clear, and we leave this as a direction for future work. 

\paragraph{Acknowledgements}

Toyota Research Institute ("TRI") provided funds and computational resources to assist the authors with their research but this article solely reflects the opinions and conclusions of its authors and not TRI or any other Toyota entity. We  thank Percy Liang and Michael Xie for helpful discussions in various stages of this work.

\medskip

{\small
\bibliographystyle{plainnat}
\bibliography{ref}}

\newpage
\appendix
\section{Missing Proofs and Derivations in Section~\ref{sec:theory}}
\label{sec:missing_theory}
\newcommand{\mrad}{\mathfrak{R}}
\newcommand{\X}{\mathcal{X}}
Let $L_{\gamma, j}$ denote the hard margin loss on examples from class $j$:
\begin{align*}
L_{\gamma,j}[f] = \Pr_{x \sim \P_j} [\max_{j'\ne j} f(x)_{j'} > f(x)_j- \gamma]
\end{align*}
and let $\hat{L}_{\gamma, j}$ denote its empirical variant. For a hypothesis class $\mathcal{F}$, let $\hat{\mathfrak{R}}_j(\mathcal{F})$ denote the empirical Rademacher complexity of its class $j$ margin: 
\begin{align*}
\hat{\mathfrak{R}}_j(\mathcal{F}) = \frac{1}{n_j}\E_{\sigma}\left[\sup_{f \in\mathcal{F}}\sum_{i\in S_j}\sigma_{i} [f(x_i)_j - \max_{j' \ne j} f(x_i)_{j'}]\right]
\end{align*}
where $\sigma$ is a vector of i.i.d. uniform $\{-1, +1\}$ bits. The following formal versiom of Theorem~\ref{thm:informal_main} bounds the balanced-class generalization $\P_{\textup{bal}}$ using samples from $\P$.
\begin{theorem}\label{thm:technical_main}
	With probability $1 - \delta$ over the randomness of the training data, for all choices of class-dependent margins $\gamma_1,\ldots, \gamma_k > 0$, all hypotheses $f \in \mathcal{F}$ will have balanced-class generalization bounded by 
	\begin{align*}
	\Lerror[f] \le \frac{1}{k} \left(\sum_{j = 1}^k \hat{L}_{\gamma_j, j}[f] + \frac{4}{\gamma_j} \hat{\mrad}_j(\mathcal{F})+ \epsilon_j(\gamma_j)\right)
	\end{align*}
	where $\epsilon_j(\gamma) \triangleq \sqrt{\frac{\log \log_2 (\frac{2 \max_{x \in\X,f \in\cF} |f(x)|}{\gamma}) + \log \frac{2c}{\delta}}{n_j}}$ is typically a low-order term in $n_j$. Concretely, the Rademacher complexity $\hat{\mathfrak{R}}_{j}(\mathcal{F})$ will typically scale as $\sqrt{\frac{\comp(\cF)}{n_j}}$ for some complexity measure $\comp(\cF)$, in which case 
	\begin{align*}
	\Lerror[f] \le \frac{1}{k} \left(\sum_{j= 1}^k \hat{L}_{\gamma_j, j}[f] + \frac{4}{\gamma_j} \sqrt{\frac{\comp(\cF)}{n_j}} + \epsilon_j(\gamma_j)\right)
	\end{align*}
\end{theorem}

\begin{proof}
	
	We will prove generalization separately for each class $j$ and then union bound over all classes. 
	
	Let $L_j[f]$ denote the test $0-1$ error of classifier $f$ on examples drawn from $\P_j$. As the examples for class $j$ is a set of $n_j$ i.i.d. draws from the conditional distribution $\P_j$, we can apply the standard margin-based generalization bound (Theorem 2 of~\citep{kakade2009complexity}) to obtain with probability $1 - \delta/c$, for all choices of $\gamma_j > 0$ and $f \in \mathcal{F}$, 
	\begin{align}
	L_{j}[f] \le \hat{L}_{\gamma_j, j} + \frac{4}{\gamma_j} \hat{\mathfrak{R}}_{j}(\mathcal{F}) + \sqrt{\frac{\log \log_2(\frac{2\max_{x \in \mathcal{X},f \in\mathcal{F}}|f(x)|}{\gamma_j})}{n_j}} + \sqrt{\frac{\log \frac{2c}{\delta}}{n_j}} \label{eq:balanced_gen:1}
	\end{align}
	Now since
	$\Lerror = \frac{1}{k} \sum_{j= 1}^k L_j$, we can union bound over all classes and average~\eqref{eq:balanced_gen:1}~to get the desired result.
\end{proof}

\newcommand{\bias}{b}
\newcommand{\msum}{\beta}
We will now show that in the case of $k = 2$ classes, it is always possible to shift the margins in order to optimize the generalization bound of~Theorem~\ref{thm:technical_main}~by adding bias terms.
\begin{theorem}
	For binary classification, let $\mathcal{F}$ be a hypothesis class of neural networks with a bias term, i.e. $\mathcal{F} = \{f + b\}$ where $f$ is a neural net function and $b \in \R^2$ is a bias, with Rademacher complexity upper bound $\hat{\mrad}_j(\mathcal{F}) \le \sqrt{\frac{\comp(\cF)}{n_j}}$. Suppose some classifier $f \in \mathcal{F}$ can achieve a total sum of margins $\gamma_1' + \gamma_2' = \msum$ with $\gamma_1', \gamma_2' > 0$. Then there exists a classifier $f^\star \in \mathcal{F}$ with margins 
	\begin{align*}
		\gamma_1^\star = \frac{\msum n_2^{1/4}}{n_1^{1/4} + n_2^{1/4}} ~~, \gamma_2^\star = \frac{\msum n_1^{1/4}}{n_1^{1/4} + n_2^{1/4}}
	\end{align*}
	which with probability $1 - \delta$ obtains the optimal generalization guarantees for Theorem~\ref{thm:technical_main}:
	\begin{align*}
	\Lerror[f^\star] \le \min_{\gamma_1 + \gamma_2 = \beta} \left(\frac{2}{\gamma_1}\sqrt{\frac{\comp(\cF)}{n_1}} + \frac{2}{\gamma_2} \sqrt{\frac{\comp(\cF)}{n_2}} \right) + \epsilon(\gamma^\star_1) + \epsilon(\gamma^\star_2)
	\end{align*}
	where $\epsilon$ is defined in Theorem~\ref{thm:technical_main}. Furthermore, this $f^\star$ is obtained via $f + b^\star$ for some bias $b^\star$. 
\end{theorem}
\begin{proof}
	For our bias $b^\star$, we simply choose $b^\star_1 = (\gamma_1^\star - \gamma_1')/2$, $b^\star_2 = -(\gamma_1^\star - \gamma_1')/2$. Now note that adding a bias term simply shifts the margins for class 1 by $b^\star_1 - b^\star_2$, giving a new margin of $\gamma_2^\star$. Likewise, the margin for class 2 becomes 
	\begin{align*}
		b^\star_2 - b^\star_1 + \gamma_2' = \gamma_2' - \gamma_1^\star + \gamma_1' = \beta - \gamma_1^\star = \gamma_2^\star
	\end{align*} 
	Now we apply Theorem~\ref{thm:technical_main} to get with probability $1 - \delta$ the generalization error bound
	\begin{align*}
		\Lerror[f^\star] \le \frac{2}{\gamma_1^\star}\sqrt{\frac{\comp(\cF)}{n_1}} + \frac{2}{\gamma_2^\star} \sqrt{\frac{\comp(\cF)}{n_2}} + \epsilon(\gamma^\star_1) + \epsilon(\gamma^\star_2)
	\end{align*}
	To see that $\gamma_1^\star, \gamma_2^\star$ indeed solve 
	\begin{align*}
		\min_{\gamma_1 + \gamma_2 = \beta} \frac{1}{\gamma_1} \sqrt{\frac{1}{n_1}} + \frac{1}{\gamma_2}\sqrt{\frac{1}{n_2}}
	\end{align*}
	we can substitute $\gamma_2 = \beta - \gamma_1$ into the expression and set the derivative to 0, obtaining
	\begin{align*}
		\frac{1}{(\beta- \gamma_1)^2 \sqrt{n_2}}-\frac{1}{\gamma_1^2\sqrt{n_1}} = 0
		\end{align*}
		Solving gives $\gamma_1^\star$. 
\end{proof}

\section{Implementation details} \label{sec:implementation}

\begin{figure}
     \centering
     \begin{subfigure}[b]{0.32\textwidth}
         \centering
         \includegraphics[width=\textwidth]{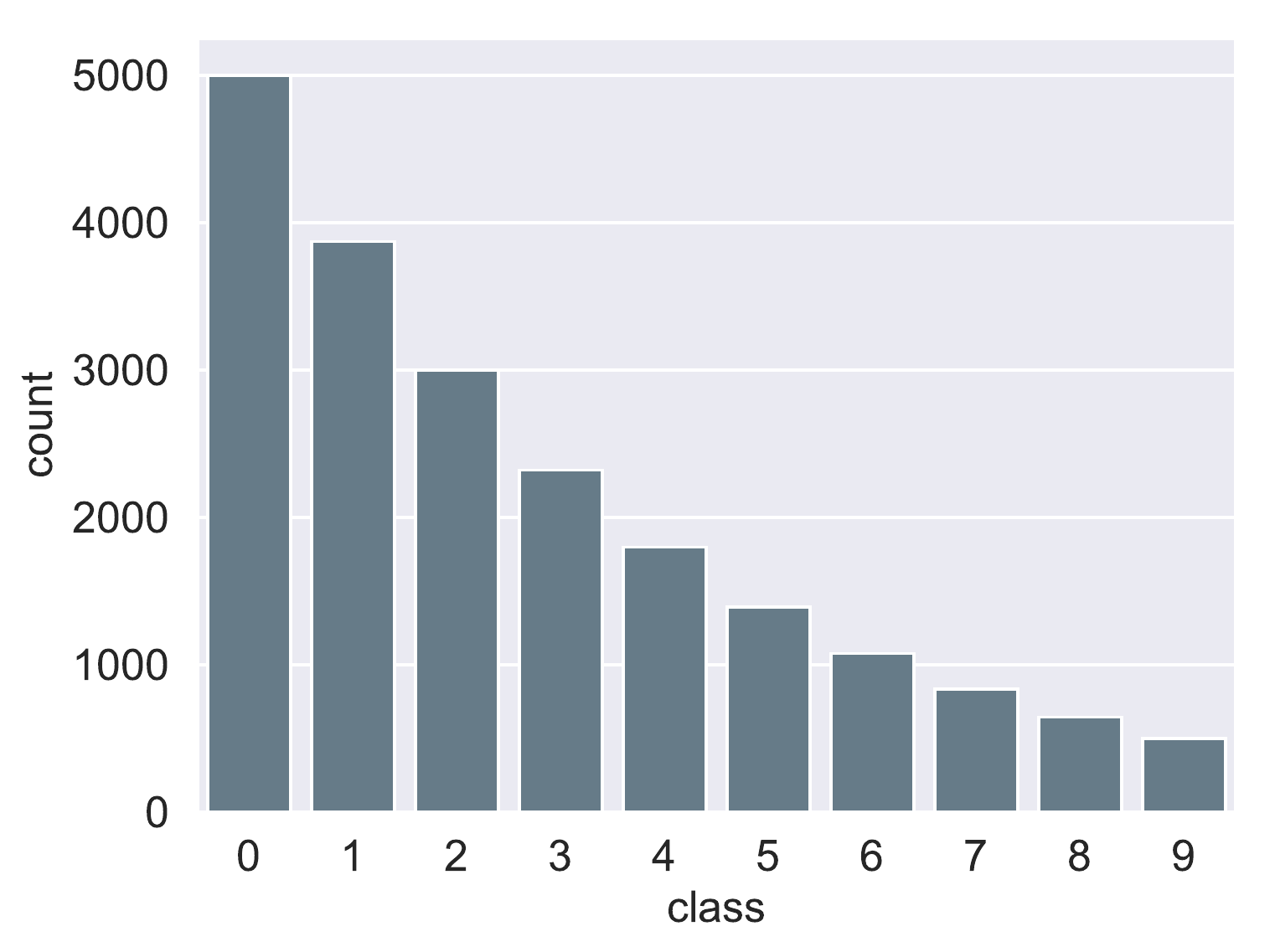}
         \caption{$\rho=10$}
         \label{fig:exp_10}
     \end{subfigure}
     \hfill
     \begin{subfigure}[b]{0.32\textwidth}
         \centering
         \includegraphics[width=\textwidth]{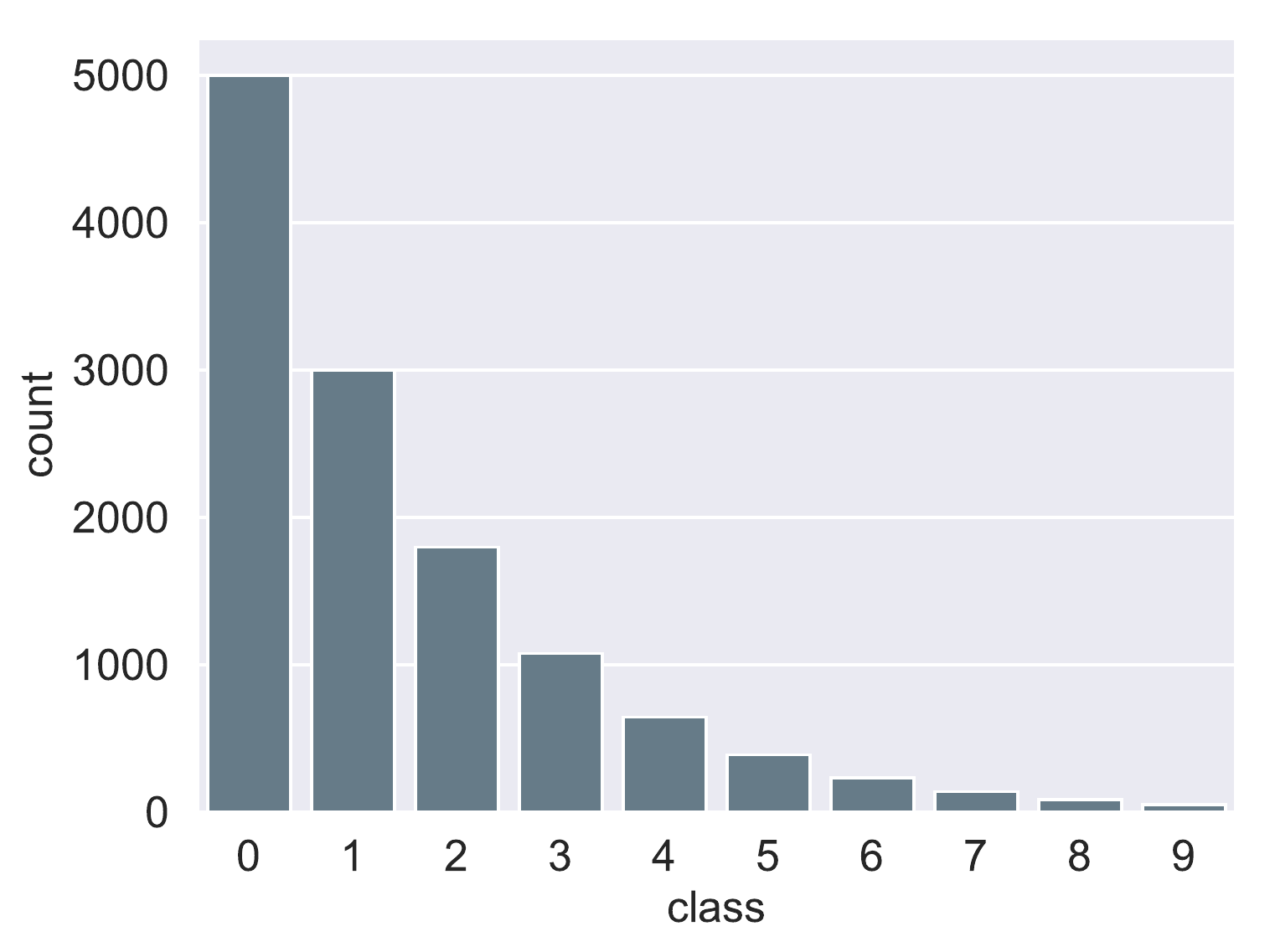}
         \caption{$\rho=100$}
         \label{fig:exp_100}
     \end{subfigure}
     \hfill
     \begin{subfigure}[b]{0.32\textwidth}
         \centering
         \includegraphics[width=\textwidth]{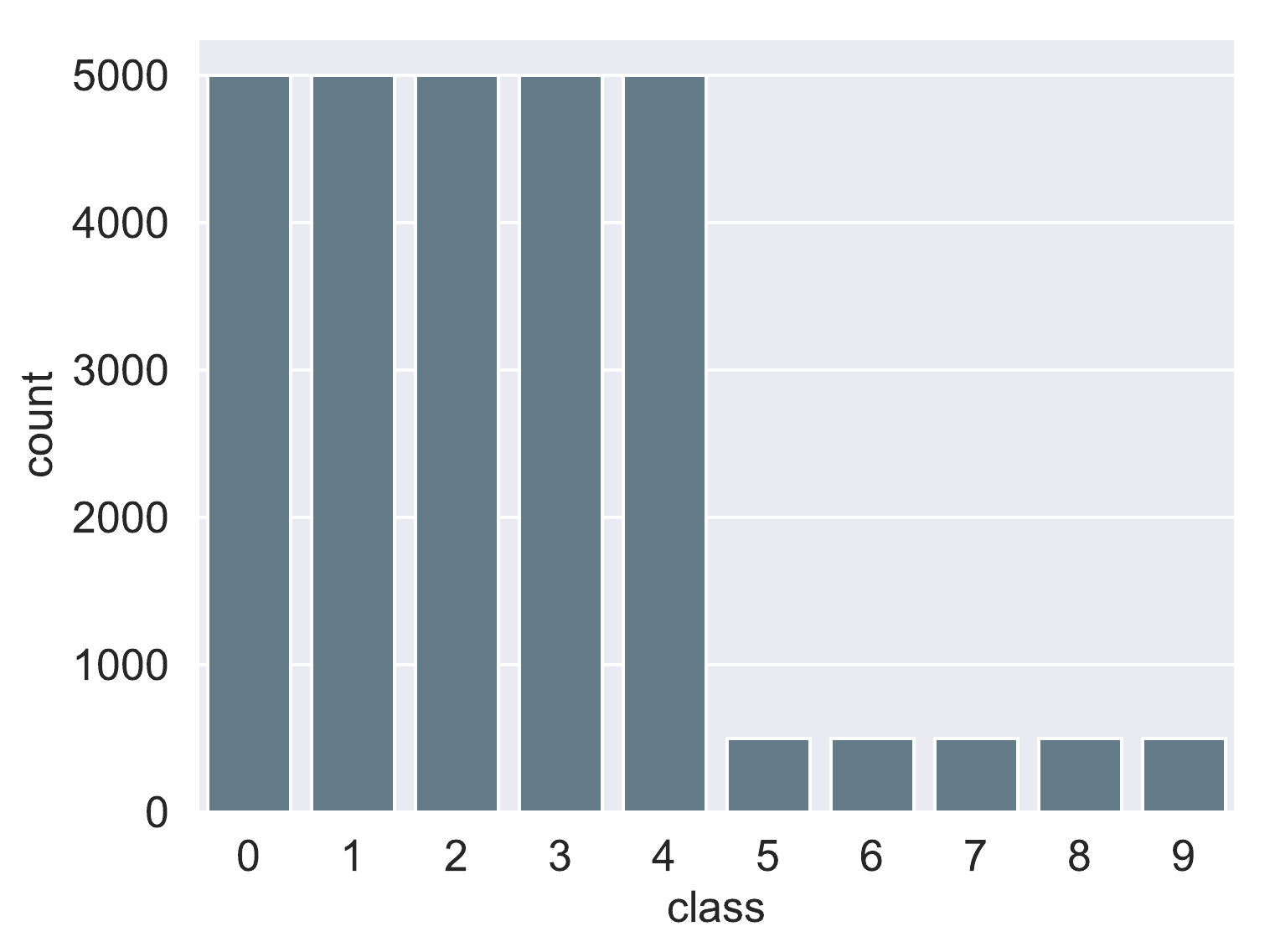}
         \caption{$\rho=10,\ \mu=0.5$}
         \label{fig:hard_10}
     \end{subfigure}
        \caption{Number of training examples per class in artificially created imbalanced CIFAR-10 datasets. Fig.~\ref{fig:exp_10} and Fig.~\ref{fig:exp_100} belong to long-tailed imbalance type and Fig.~\ref{fig:hard_10} is a step imbalance distribution.}
        \label{fig:imb_example}
\end{figure}

\textbf{Label distributions.} Some example distributions of our artificially created imbalance are shown in Figure~\ref{fig:imb_example}.

\textbf{Implementation details for CIFAR.} For CIFAR-10 and CIFAR-100, we follow the simple data augmentation in \citep{he2016deep} for training: 4 pixels are padded on each side, and a $32 \times 32$ crop is randomly sampled from the padded image or its horizontal flip. We use ResNet-32 \citep{he2016deep} as our base network, and use stochastic gradient descend with momentum of 0.9, weight decay of $2\times 10^{-4}$ for training. The model is trained with a batch size of 128 for 200 epochs. For fair comparison, we use an initial learning rate of 0.1, then decay by 0.01 at the 160th epoch and again at the 180th epoch. We also use linear warm-up learning rate schedule \citep{goyal2017accurate} for the first 5 epochs for fair comparison. Notice that the warm-up trick is essential for the training of re-weighting, but it won't affect other algorithms in our experiments. We tune $C$ to normalize $\Delta_j$ so that the largest enforced margin is $0.5$.

\textbf{Implementation details for Tiny ImageNet.} For Tiny ImageNet, we perform simple horizontal flips, taking random crops of size $64 \times 64$ from images padded by 8 pixels on each side. We perform 1 crop test with the validation images. We use ResNet-18 \citep{he2016deep} as our base network, and use stochastic gradient descend with momentum of 0.9, weight decay of $2\times 10^{-4}$ for training. We train the model using a batch size of 128 for 120 epochs with a initial learning rate of 0.1. We decay the learning rate by 0.1 at epoch 90. We tune $C$ to normalize $\Delta_j$ so that the largest enforced margin is $0.5$.

\textbf{Implementation details for iNaturalist 2018.} On iNaturalist 2018, we followed the same training strategy used by \citep{he2016deep} and trained ResNet-50 with 4 Tesla V100 GPUs. Each image is first resized by setting the shorter side to 256 pixels, and then a $224 \times 224$ crop is randomly sampled from an image or its horizontal flip. We train the network for 90 epochs with an initial learning rate of 0.1. We anneal the learning rate at epoch 30 and 60. For our two-stage training schedule, we rebalance the training data starting from epoch 60. We tune $C$ to normalize $\Delta_j$ so that the largest enforced margin is $0.3$.

\section{Additional Results}

\subsection{Feature visualization} 

To have a better understanding of our proposed LDAM loss, we use a toy example to visualize feature distributions trained under different schemes. We train a 7-layer CNN as adopted in~\citep{liu2017rethinking} on MNIST~\citep{lecun1998gradient} with step imbalance setting ($\rho=100, \mu=0.5$). For a more intuitive visualization, we constrain the feature dimension to 3 and normalize the feature before feeding it into the final fully-connected layer, allowing us to scatter the features on a unit hyper-sphere in a 3D frame. The visualization is shown in Figure~\ref{fig:visualize} with additional discussion in the caption. %

\begin{figure}
	\centering
	\begin{subfigure}[b]{0.24\textwidth}
		\centering
		\includegraphics[width=\textwidth]{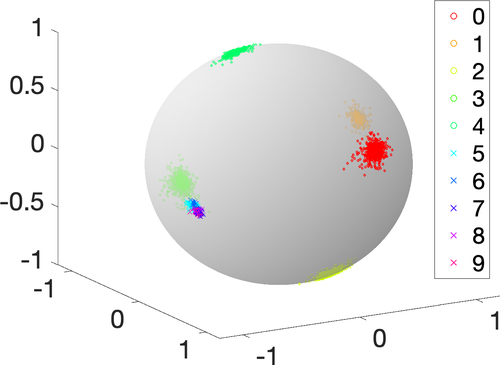}
		\caption{\ERM{} train}
		\label{fig:ERM_train}
	\end{subfigure}
	\hfill
	\begin{subfigure}[b]{0.24\textwidth}
		\centering
		\includegraphics[width=\textwidth]{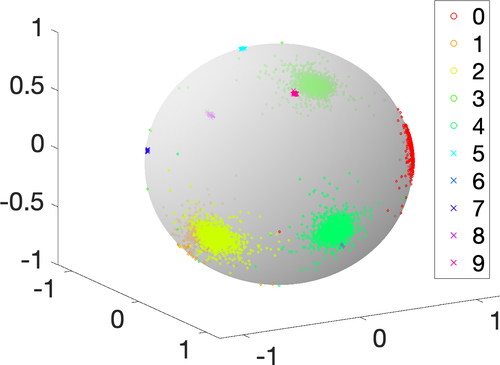}
		\caption{Re-sampling train}
		\label{fig:resample_train}
	\end{subfigure}
	\hfill
	\begin{subfigure}[b]{0.24\textwidth}
		\centering
		\includegraphics[width=\textwidth]{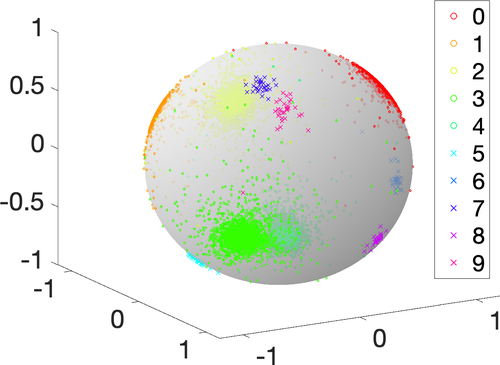}
		\caption{Re-weighting train}
		\label{fig:reweight_train}
	\end{subfigure}
	\hfill
	\begin{subfigure}[b]{0.24\textwidth}
		\centering
		\includegraphics[width=\textwidth]{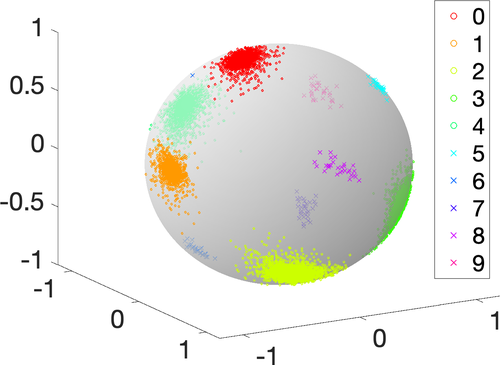}
		\caption{\BMCE{} train}
		\label{fig:BM_train}
	\end{subfigure}
	\\
	\begin{subfigure}[b]{0.24\textwidth}
		\centering
		\includegraphics[width=\textwidth]{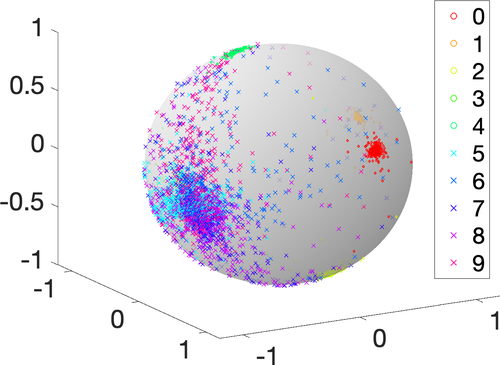}
		\caption{\ERM{} val}
		\label{fig:ERM_test}
	\end{subfigure}
	\hfill
	\begin{subfigure}[b]{0.24\textwidth}
		\centering
		\includegraphics[width=\textwidth]{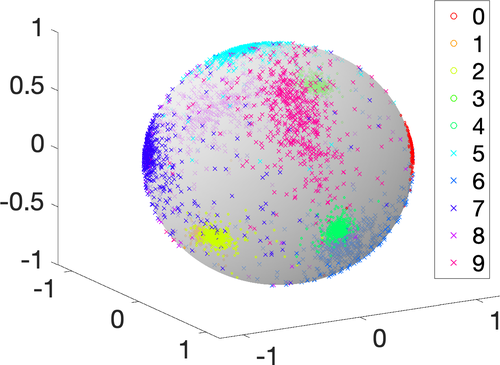}
		\caption{Re-sampling val}
		\label{fig:resample_test}
	\end{subfigure}
	\hfill
	\begin{subfigure}[b]{0.24\textwidth}
		\centering
		\includegraphics[width=\textwidth]{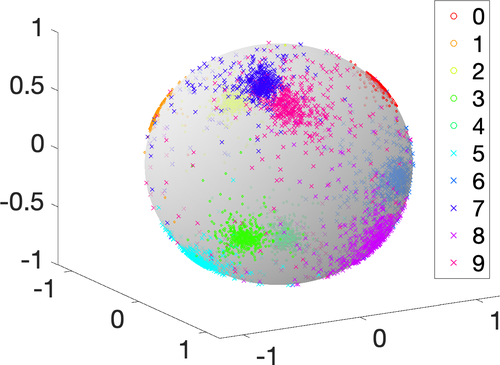}
		\caption{Re-weighting val}
		\label{fig:reweight_test}
	\end{subfigure}
	\hfill
	\begin{subfigure}[b]{0.24\textwidth}
		\centering
		\includegraphics[width=\textwidth]{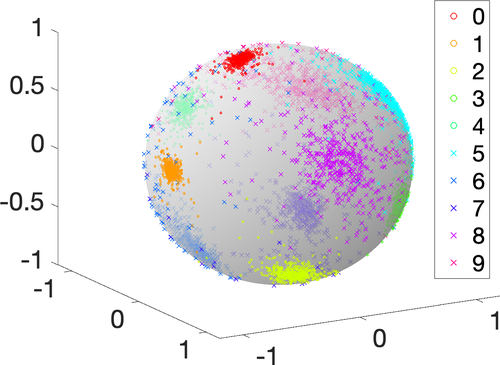}
		\caption{\BMCE{} val}
		\label{fig:BM_test}
	\end{subfigure}
	\caption{Visualization of feature distribution of different methods. We constrain the feature dimension to be three and normalize it for better illustration. The top row has the feature distribution on the training set and the second row the feature distributions on the validation set. We can see that \ldam{} appears to have more separate training features compared to the other methods. We note this visualization is only supposed to provide qualitative intuitions, and the differences between our methods and other methods may be more significant for harder tasks with higher feature dimension. (For example, here the accuracies of re-weighting and \ldam{} are very similar, whereas for large-scale datasets with higher feature dimensions, the gap is significantly larger.)}
	\label{fig:visualize}
\end{figure}

\subsection{Visual Recognition on imbalanced Tiny ImageNet}
\label{sec:tinyimagenet}

\begin{table}[]
	\caption{Validation error on imbalanced Tiny ImageNet with different loss functions and training schedules.}
	\label{tab:tiny-imagenet-table}
	\centering
	\begin{tabular}{cc|cc|cc|cc|cc}
		\toprule
		\multicolumn{2}{c|}{Imbalance Type}  & \multicolumn{4}{c|}{long-tailed}                     & \multicolumn{4}{c}{step}                  \\ \midrule
		\multicolumn{2}{c|}{Imbalance Ratio} & \multicolumn{2}{c|}{100} & \multicolumn{2}{c|}{10} & \multicolumn{2}{c|}{100} & \multicolumn{2}{c}{10} \\ \midrule
		Loss & Schedule          & Top-1       & Top-5       & Top-1       & Top-5      & Top-1       & Top-5       & Top-1       & Top-5      \\ \midrule
		ERM & SGD & 66.19 & 42.63 & 50.33 & 26.68 & 63.82 & 44.09 & 50.89 &   27.06 \\
		CB SM & SGD & 72.72 & 52.62 & 51.58 & 28.91 & 74.90 & 59.14 & 54.51 &    33.23    \\
		ERM & \tstagew{} & 64.57 & 40.79 & 50.03 & 26.19 & 62.36 & 40.84    & 49.17  & 25.91 \\    
		\BMCE{} & SGD & 64.04 & 40.46  & 48.08 & 24.80 & 62.54 & 39.27 & 49.08 & 24.52 \\
		\BMCE{} & \tstagew{} & \textbf{62.53} & \textbf{39.06} & \textbf{47.22} & 
		\textbf{23.84} & \textbf{60.63} & \textbf{38.12} & \textbf{47.43} & \textbf{23.26} \\ \bottomrule
	\end{tabular}
\end{table}

In addition to artificial imbalanced CIFAR, we further verify the effectiveness of our method on artificial imbalanced Tiny ImageNet. The Tiny ImageNet dataset has 200 classes. Each class has 500 training images and 50 validation images of size $64 \times 64$. We use the same strategy described above to create long-tailed and step imbalance versions of Tiny ImageNet. The results are presented in Table~\ref{tab:tiny-imagenet-table}. While Class-Balanced Softmax performs worse than the \ERM{} baseline, the proposed \BMCE{} and \tstagew{} demonstrate consistent improvements over \ERM{}.

\subsection{Comparing feature extractors trained by different schemes} \label{sec:feature_extractor}

\begin{figure}
     \centering
     \includegraphics[width=0.65\textwidth]{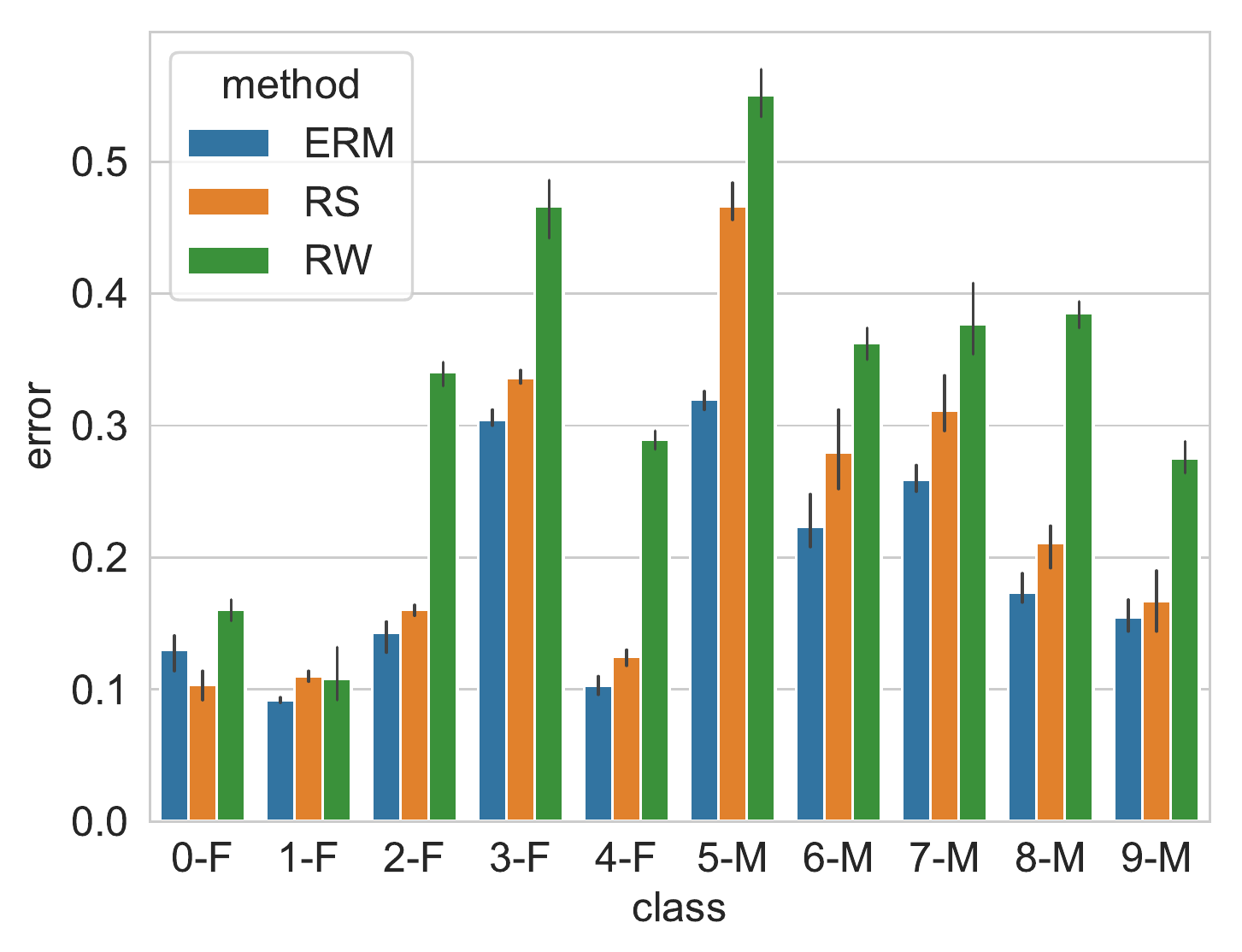}
     \caption{In the setting of training mbalanced CIFAR-10 dataset with step imbalance of $\rho=100, \mu=0.5$, to test the quality of the features obtained by the ERM, RW and RS before annealing the learning rate, we use a subset of the \textit{balanced} validation dataset to train linear classifiers on top of the features, and evaluate the per-class validation error on the rest of the validation data. (Little over-fitting in training the linear classifier is observed.) The left-5 classes are frequent and denoted with -F. The features obtained from ERM setting has the strongest performance, confirming our intuition that the second stage of \tstagew{} starts from better features. In the second stage, \tstagew{} re-weights the example again, adjusting the decision boundary and locally fine-tuning the features.}
     \label{fig:linear_finetune}
\end{figure}

As discussed in Section~\ref{sec:ablation}, we train a linear classifier on features extracted by backbone filters pretrained under different schemes. We could conclude that for highly imbalanced settings (step imbalance with $\rho=100, \mu=0.5$), backbone networks trained by \ERM{} learns the most expressive feature embedding compared with the other two methods, as shown in Figure~\ref{fig:linear_finetune}.

\subsection{Comparing DRW and DRS}

Our proposed deferred re-balancing optimization schedule can be combined with either re-weighting or re-sampling. We use re-weighting as the default choice in the main paper. Here we demonstrate through Table~\ref{tab:2-stage-table} that re-weighting and re-sampling exhibit similar performance when combined with deferred re-balancing scheme. This result could be explained by the fact that the second stage does not move the weights far. Re-balancing in the second stage mostly re-adjusts the decision boundary and thus there is no significant difference between using re-weighting or re-sampling for the second stage.

\begin{table}[]
\caption{Top-1 validation error of ResNet-32 trained with different training schedules on imbalanced CIFAR-10 and CIFAR-100.}
\label{tab:2-stage-table}
\centering
\begin{tabular}{c|cc|cc|cc|cc}
\toprule
Dataset Name           & \multicolumn{4}{c|}{Imbalanced CIFAR-10}                               & \multicolumn{4}{c}{Imbalanced CIFAR-100}                              \\ \midrule
Imbalance Type         & \multicolumn{2}{c|}{long-tailed} & \multicolumn{2}{c|}{step} & \multicolumn{2}{c|}{long-tailed} & \multicolumn{2}{c}{step} \\ \midrule
Imbalance Ratio        & \multicolumn{1}{c|}{100}  & 10 & \multicolumn{1}{c|}{100}   & 10   & \multicolumn{1}{c|}{100}  & 10 & \multicolumn{1}{c|}{100}   & 10   \\ \midrule
ERM & 29.64& 13.61& 36.70 & 17.50 & 61.68& 44.30 & 61.05& 45.37 \\
DRW & 25.14 & 13.12 & 28.40 & 14.49 & 59.34 & 42.68 & 58.86                            & 42.78      \\
DRS & 25.50 & 13.28 & 27.97 & 14.83 & 59.67 &  42.74   &   58.65                          & 43.21      \\\bottomrule
\end{tabular}
\end{table}

\begin{figure}
     \centering
     \begin{subfigure}[b]{0.32\textwidth}
         \centering
         \includegraphics[width=\textwidth]{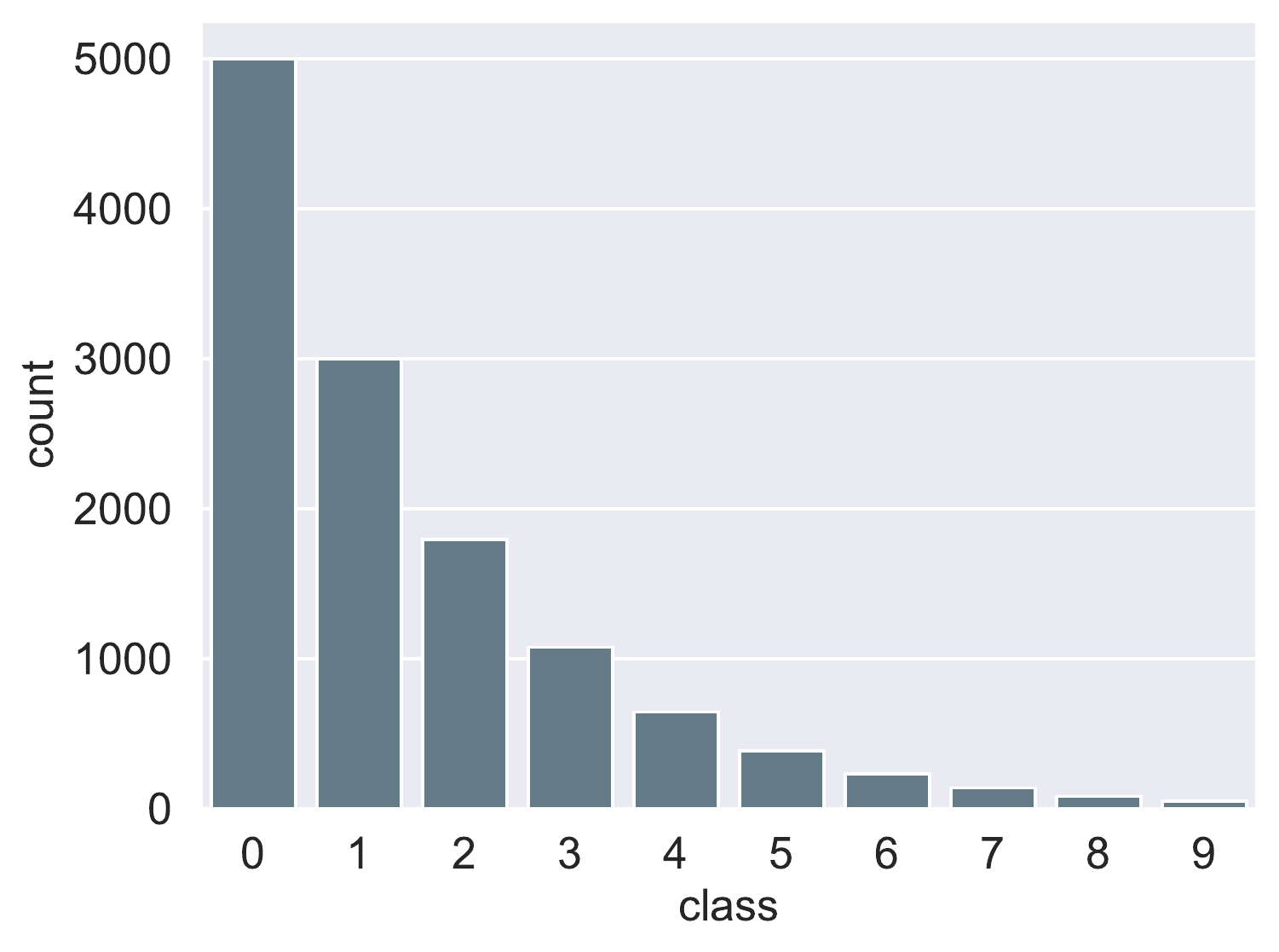}
         \caption{train set distribution}
         \label{fig:exp_train}
     \end{subfigure}
     \hfill
     \begin{subfigure}[b]{0.32\textwidth}
         \centering
         \includegraphics[width=\textwidth]{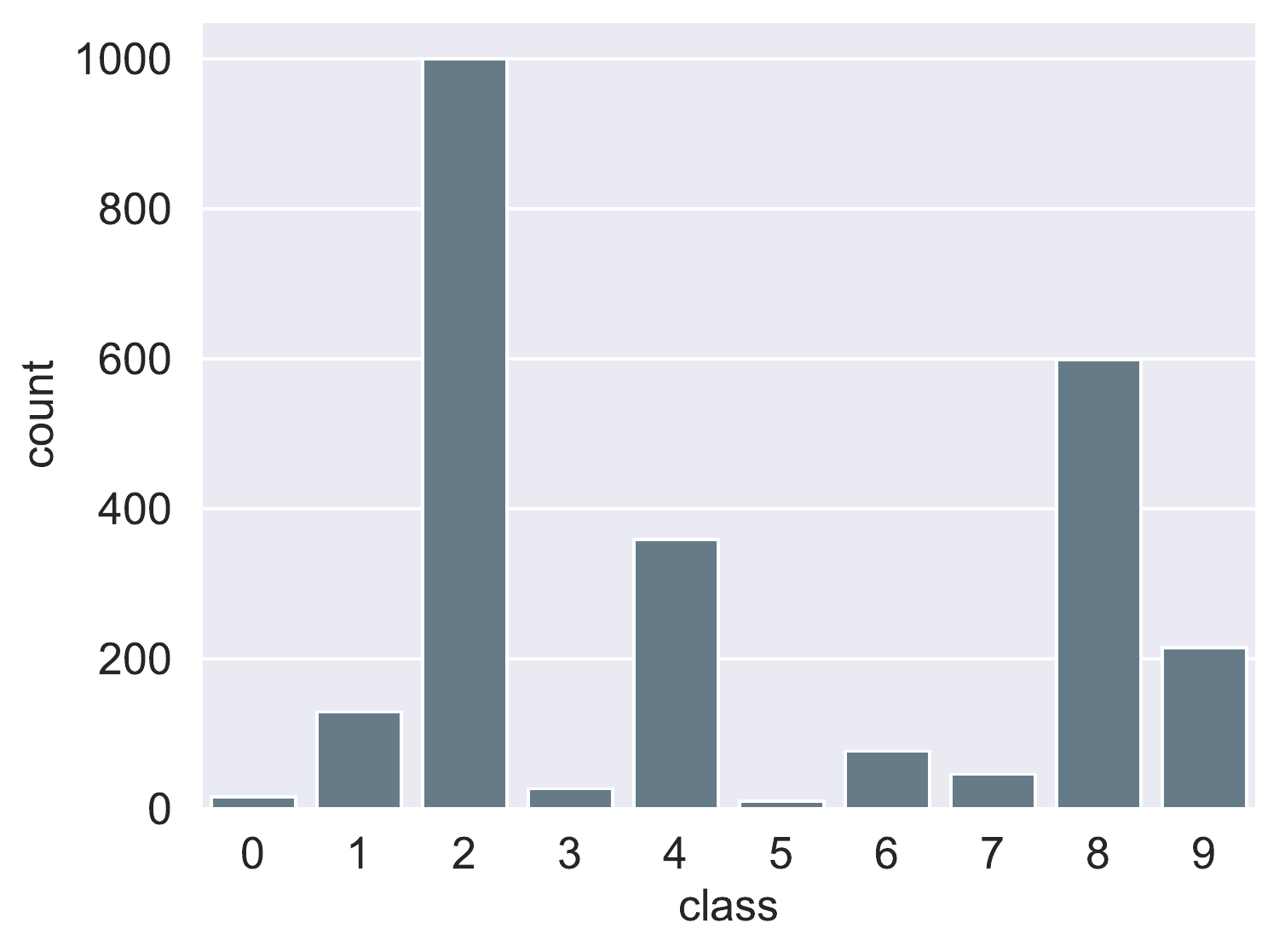}
         \caption{val set distribution 1}
         \label{fig:exp_val1}
     \end{subfigure}
     \hfill
     \begin{subfigure}[b]{0.32\textwidth}
         \centering
         \includegraphics[width=\textwidth]{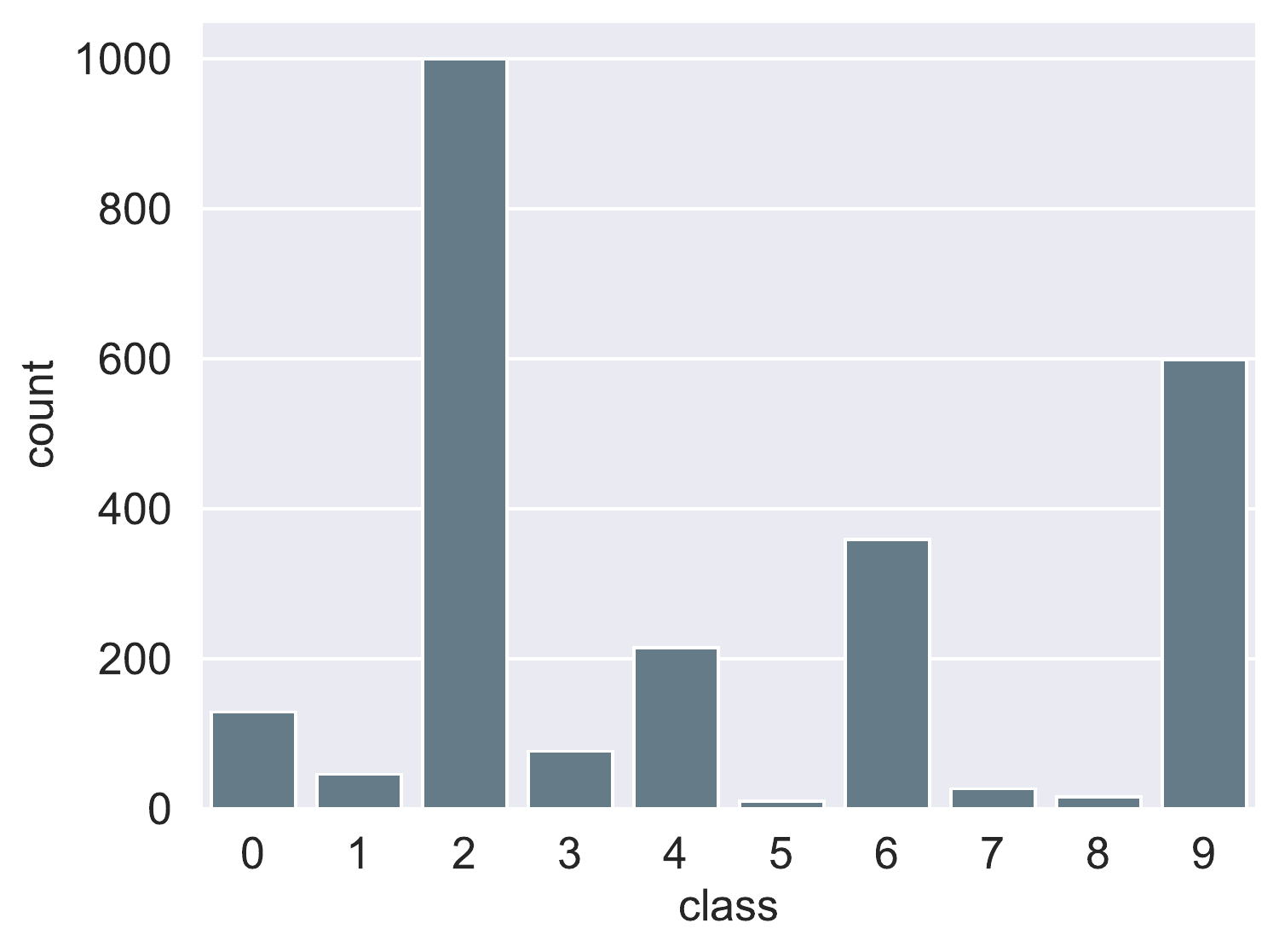}
         \caption{val set distribution 2}
         \label{fig:exp_val2}
     \end{subfigure}
        \caption{Example distributions when train and test distributions are both imbalanced. As discussed in~\ref{sec:imbalanced_test} we run two random seeds for generating test distributions. Here Figure~\ref{fig:exp_val1} denotes the left column in Table~\ref{tab:imbalanced_test}.}
        \label{fig:imb_test}
\end{figure}

\begin{table}[htbp]
\caption{Top-1 validation error of ResNet-32 on imbalanced training and imbalanced validation scheme for CIFAR-10. See Section~\ref{sec:imbalanced_test} for details.}
\label{tab:imbalanced_test}
\centering
\begin{tabular}{c|cc|cc|cc|cc}
\toprule
Imbalance Type  & \multicolumn{4}{c|}{long-tailed} & \multicolumn{4}{c}{step} \\ \midrule
Imbalance Ratio Train  & \multicolumn{2}{c|}{100}  & \multicolumn{2}{c|}{10}   & \multicolumn{2}{c|}{100} & \multicolumn{2}{c}{10}   \\ \midrule
Imbalance Ratio Val  & \multicolumn{1}{c|}{100}  & 100 & \multicolumn{1}{c|}{10}   & 10   & \multicolumn{1}{c|}{100}  & 100 & \multicolumn{1}{c|}{10}   & 10   \\ \midrule
ERM & 30.99 & 28.45 & 13.08 & 13.12 & 24.55 & 28.63 & 10.34 & 11.67 \\
CB-RW & 20.86 & 26.19 & 10.70 & 11.93 & 35.76 & 31.35 & 9.82 & 11.02      \\
LDAM-DRW &  \textbf{14.40} & \textbf{12.95} & \textbf{10.12} & \textbf{10.62} & \textbf{10.30} & \textbf{9.54} & \textbf{7.51} & \textbf{7.82}  \\\bottomrule
\end{tabular}
\end{table}

\subsection{Imbalanced Test Label Distributions}\label{sec:imbalanced_test}

Though the majority of our experiments follow the uniform test distribution setting, it could be extended to imbalanced test distribution naturally. Suppose the number of training examples in class $i$ is denoted by $n_i$ and the number of test examples in class $i$ is denoted by $n_i'$, then we could adapt the \ldam{} simply by encouraging the margin $\Delta_i$ for class $i$ with 
\begin{align}
\Delta_j \propto \left(\frac{n_i'}{n_i}\right)^{1/4}
\end{align}

To complement our main result, In Table~\ref{tab:imbalanced_test}, we demonstrate that this extended algorithm can also work well when the test distribution is imbalanced. We use the same rule as described in Section~\ref{sec:experiments} to generate imbalanced test label distribution and then permute randomly the frequency of the labels (so that the training label distribution is very different from the test label distribution.). For example, in the experiment shown in Figure~\ref{tab:imbalanced_test}, the training label distribution of the column of "long-tailed with $\rho=100$" follows Figure~\ref{fig:exp_train} (which is the same as Figure~\ref{fig:exp_100}) whereas the test label distribution is shown in Figure~\ref{fig:exp_val1} and Figure~\ref{fig:exp_val2}. For each of the settings reported in Table~\ref{tab:imbalanced_test}, we have run it with two different random seeds for generating the test label distribution, and we see qualitatively similar results. We refer to our code for the precise label distribution generated in the experiments.\footnote{Code available at \url{https://github.com/kaidic/LDAM-DRW}.}

\end{document}